\documentclass[letterpaper]{article} 
\usepackage[preprint]{aaai2027}  
\usepackage[hyphens]{url}  
\usepackage{graphicx} 
\usepackage{natbib}  
\usepackage{caption} 
\usepackage{algorithm}
\usepackage{algorithmic}
\usepackage{amsmath,amsfonts,amssymb,amsthm}
\usepackage{multirow,xcolor}
\usepackage[table]{xcolor}
\newtheorem{proposition}{Proposition}

\usepackage{newfloat}
\usepackage{listings}
\DeclareCaptionStyle{ruled}{labelfont=normalfont,labelsep=colon,strut=off} 
\floatstyle{ruled}
\newfloat{listing}{tb}{lst}{}
\floatname{listing}{Listing}

\usepackage{booktabs}
\usepackage{algorithm}
\usepackage{algorithmic}
\usepackage{amsmath,amsfonts,amssymb,amsthm}
\usepackage[export]{adjustbox}
\usepackage{xcolor}
\title{Class Geometry as Supervision for Sample-Efficient Open-World Detection}
\author{
    Akash Rao\textsuperscript{\rm 1}\corresponding, Zhou Chen\textsuperscript{\rm 1}\equalcontrib, Revanth Reddy Palem\textsuperscript{\rm 1}\equalcontrib, Udhav Ramachandran\textsuperscript{\rm 2}, Ruth Scimeca\textsuperscript{\rm 2}, Sathyanarayanan N. Aakur\textsuperscript{\rm 1}\corresponding
}
\affiliations{
    \textsuperscript{\rm 1}CSSE Department, Auburn University, Auburn, AL, 36849\\

    \textsuperscript{\rm 2} Department of Veterinary Pathobiology, Oklahoma State University, Stillwater, OK, 74078\\
    \{azr0187,san0028\}@auburn.edu
}

\begin{document}

\maketitle

\begin{abstract}
Open-world object detection requires models to recognize known categories, reject unfamiliar objects, and incorporate new classes over time. This is especially challenging in scarce-data settings such as biomedical and scientific imaging, where rare categories may have only a few annotated examples and fine-grained classes differ by subtle morphology. Prototype-based detectors are natural for this regime, but they typically learn class prototypes as independent anchors, ignoring relational structure among classes. We propose class-geometry supervision (CGS), a general framework that constrains learned prototype or class-representation spaces to preserve visual or semantic class dissimilarities estimated from training data. CGS introduces a dissimilarity-preserving objective that aligns pairwise distances among learned class representations with a target class-geometry matrix while retaining the standard task loss. We instantiate the same objective across prototype recognition, few-shot biomedical object detection, open-set detection, novel-class insertion, and OWOD adaptation on COCO. Experiments show that CGS improves sample efficiency in recognition and ova detection, substantially strengthens novel-class insertion, and improves unknown recall on COCO while retaining much of the known-class detection performance. Ablations show that meaningful visual geometry provides the most reliable gains, while random geometry can help novel separation but is less consistent for few-shot detection. These results suggest that relational class geometry is an effective supervisory signal for building calibrated and extensible open-world detectors under limited supervision.
\end{abstract}


\section{Introduction}

Open-world object detection requires a model to detect known categories, identify unfamiliar objects, and incorporate new classes over time. This capability is especially important in biomedical and scientific imaging, where rare categories, emerging subtypes, and expert-defined labels often appear with only a handful of annotated examples. In these settings, detection cannot rely on large-scale retraining whenever new classes arise. Instead, models must support sample-efficient open-world detection: learning known categories from scarce examples, recognizing when objects fall outside the known label set, and inserting new categories with minimal supervision. While existing open-world detectors have made substantial progress under standard benchmark protocols, they are often evaluated in data-rich regimes and provide limited guidance for scarcely labeled data.

The central difficulty is that scarce examples provide an incomplete view of each class. This is particularly problematic in fine-grained domains, where visually similar categories differ only by subtle morphology, and where unknown objects may be mistaken for rare known classes or background. Prototype-based recognition and detection offer a natural starting point because they represent each class using a small support set. However, conventional prototype learning treats class prototypes as independent anchors: each class is pulled toward its examples and pushed away from others only through task loss. This ignores a key source of structure that is often available even when labels are scarce: classes have geometry. Some categories are visually, semantically, or morphologically close, while others are distant. Without this relational structure, prototype spaces can become poorly calibrated for open-world decisions.

In this work, we propose to use \textit{class geometry as supervision} for sample-efficient open-world detection. Rather than learning prototypes solely from scarce labeled examples, we constrain the inter-class prototype space to preserve a target geometry derived from class-level visual or semantic dissimilarities. In the biomedical setting, this geometry is estimated from training examples by computing average pixel-level differences between classes, capturing coarse morphology relationships among ova; we also study language-derived dissimilarities from visual class descriptions. As shown in Fig.~\ref{fig:intuition}, this provides supervision over how categories should be arranged: visually similar classes are encouraged to remain nearby, while dissimilar classes are separated. We realize this principle through a dissimilarity-preserving prototype objective, which aligns pairwise distances among class prototypes with a target visual or semantic class-dissimilarity matrix, complementing the standard recognition losses. Because the constraint acts on relationships among class prototypes rather than on individual embeddings, it organizes the class manifold while preserving intra-class variation. 

\begin{figure*}
    \centering
    \includegraphics[width=0.9\linewidth]{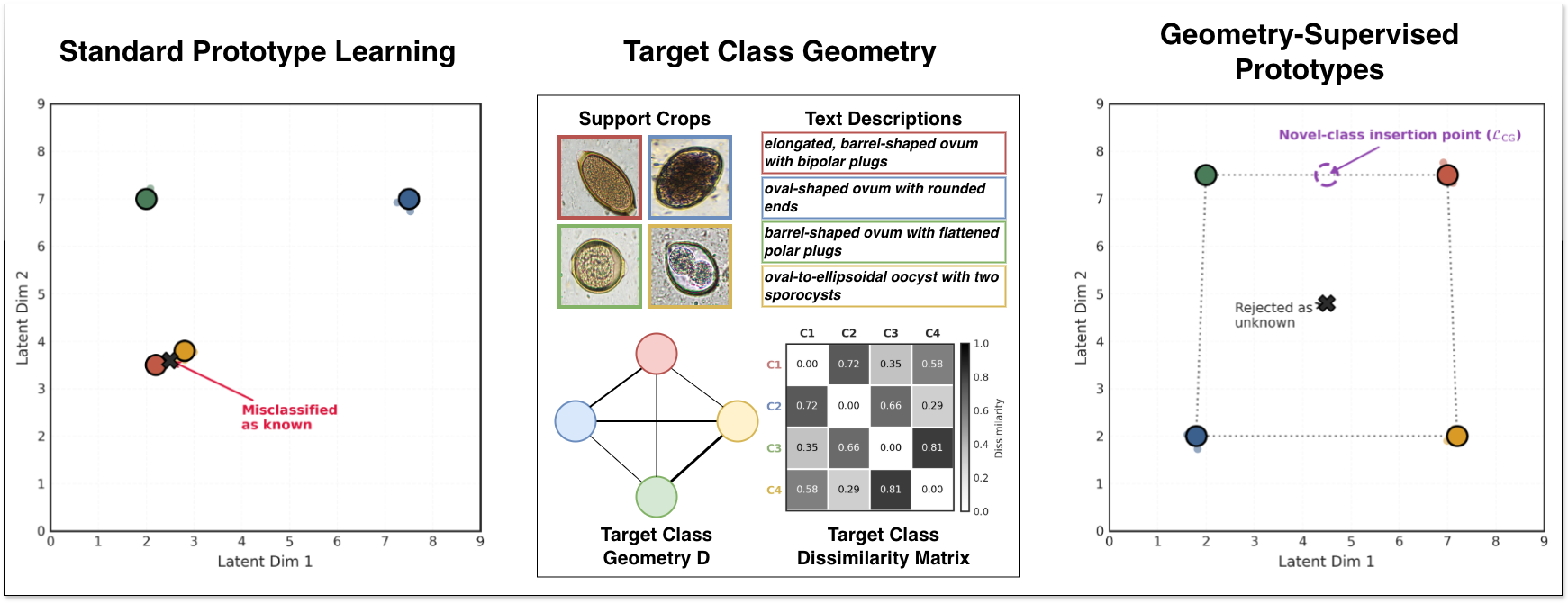}
\caption{
\textbf{Class geometry supervision.}
Standard prototype learning can produce poorly calibrated class anchors under scarce supervision, causing unknown objects to be misclassified.
We construct a training-only class geometry \(D\) from support-crop morphology and text descriptions, then use \(\mathcal{L}_{\mathrm{CG}}\) to align the learned prototype graph with this target geometry.
The resulting space better supports unknown rejection and novel-class insertion.
}
\label{fig:intuition}
\end{figure*}
We evaluate class-geometry supervision across recognition, biomedical detection, and COCO open-world detection. First, we apply CGS to ProtoNet-style recognition to isolate its effect on scarce-data prototype learning~\cite{trehan2026fsp}. Second, we integrate CGS into a prototype-based DETR detector for parasitic ova~\cite{trehan2026fsp}, where region features are matched to support-derived class prototypes for few-shot detection, open-set detection, and novel-class insertion. Third, we evaluate CGS on the standard OWOD/COCO protocol~\cite{Gupta2022} in two regimes: frozen-head adaptation over fixed OW-DETR features and full-model fine-tuning. Across these settings, CGS improves low-shot recognition and few-shot ova detection, yields large gains in novel and generalized novel-class insertion, and improves unknown recall on COCO. The COCO results reveal a tradeoff: frozen-head CGS substantially improves open-world sensitivity, while full-model fine-tuning reduces the mAP cost and preserves known-class performance.

Our contributions are four-fold:
(i) We formulate class-geometry supervision as a mechanism for sample-efficient open-world detection, where known-class detection, unknown discovery, and novel-class insertion must operate under limited labeled data.
(ii) We propose a dissimilarity-preserving prototype objective that aligns pairwise distances among learned class representations with visual or semantic class-dissimilarity structure, complementing standard matching and detection losses.
(iii) We instantiate this objective across prototype recognition and biomedical open-world detection, showing gains in low-shot recognition, few-shot detection, open-set detection, and novel-class insertion.
(iv) We evaluate CGS on COCO OWOD in frozen-head and full fine-tuning regimes, showing that class geometry improves unknown recall while exposing a controllable tradeoff with known-class mAP.

\section{Related Work}

\textbf{Open-world and open-vocabulary object detection} 
Object detection has evolved from restrictive closed-set models to advanced open-world (OWOD) and open-vocabulary (OVOD) paradigms capable of identifying unknown objects and incrementally learning new categories. While OWOD frameworks—building on foundational models like DETR \cite{Carion2020} have refined the separation of knowns and unknowns through advanced probabilistic and neutralization techniques \cite{Joseph2021,Gupta2022,zohar2023prob, xi2024ktcn}, OVOD models leverage large vision-language architectures like OWL-ViT \cite{Minderer2022} and its scaled successors \cite{minderer2023scaling} to detect novel items via text prompts, often utilizing multi-modal classifiers for improved zero-shot adaptability \cite{kaul2023multi}. However, a major limitation of both approaches is their heavy reliance on extensive labeled base data, complex pseudo-labeling, or computationally expensive end-to-end training pipelines. To overcome data barriers, we utilize class-level visual and semantic geometry to make open-world detection highly sample-efficient, seamlessly integrating new object classes under limited per-class supervision.

\textbf{Few-shot, prototype-based, and scarce-data detection}
Addressing the challenge of limited supervision, few-shot object detection (FSOD) seeks to recognize novel objects using only a handful of support examples. Foundational work in meta-learning, such as Prototypical Networks \cite{Snell2017}, established the paradigm of classifying instances based on their distance to learned class anchors or prototypes. This prototype-based approach has been extensively adapted for detection: early advancements utilized meta-learning frameworks like Meta R-CNN \cite{Yan2019} to aggregate support features, while subsequent research demonstrated the efficacy of two-stage fine-tuning \cite{Wang2020} and decoupled training pipelines like DeFRCN \cite{Qiao2021}. Additionally, methods such as FSCE \cite{Sun2021} have focused on refining the feature space through contrastive proposal encoding to better separate object classes. However, the key distinction is that prior prototype methods learn class anchors from support examples but do not explicitly constrain the relational geometry among the prototypes themselves. While few-shot detectors improve adaptation from limited support examples through meta-learning, fine-tuning, contrastive encoding, or prototype calibration, our approach structurally departs from this paradigm by supervising the prototype space using known class relationships, ensuring that scarce examples are inherently arranged in a more calibrated open-world geometry.

\textbf{Class-relational, semantic, and geometry-aware representation learning.} 
Moving beyond the assumption of independent classes, classical Multidimensional Scaling \cite{Kruskal1964} and modern contrastive learning paradigms \cite{Schroff2015,Khosla2020} structure representation spaces by mapping relationships to geometric constraints. Similarly, semantic structures have been explicitly modeled to improve generalization through zero-shot embeddings \cite{Liu2021}, hierarchy-aware losses \cite{Bertinetto2020}, and vision-language priors like CLIP \cite{Radford2021} and BioCLIP \cite{Stevens2024}. While these methods demonstrate that labels contain rich relational structure beyond one-hot identity, we build on this view by applying it directly to detector prototypes. Unlike prior open-world detectors that rely on high-data end-to-end training, or few-shot detectors that learn class anchors independently, we use class-level visual and semantic dissimilarities as explicit supervisory signals. By preserving this geometry during training, we structurally organize the prototype space for sample-efficient open-world detection.

\section{Proposed Approach}

\noindent\textbf{Problem Statement.}
We consider sample-efficient open-world detection, where a model must detect known classes, identify objects outside the current label set, and incorporate newly labeled classes over time. Let \(\mathcal{C}_t=\{1,\ldots,C_t\}\) denote the known classes at stage \(t\), and let \(\mathcal{U}_t\) denote unknown classes that may appear but are not yet annotated with class identities. For each known class \(c\in\mathcal{C}_t\), the model receives a small support set \(S_c=\{(I_{c,k},b_{c,k})\}_{k=1}^{K}\), where \(I_{c,k}\) is a support image, \(b_{c,k}\) is a support region or bounding box, and \(K\) is small. A detector produces candidate regions or object queries \(b\), and a region encoder maps each region to a feature \(z=f_\theta(I,b)\in\mathbb{R}^d\). Each known class is represented by a support-derived prototype \(p_c=|S_c|^{-1}\sum_{(I,b)\in S_c} f_\theta(I,b)\). A query region is classified by its similarity or distance to the current prototype set \(\mathcal{P}_t=\{p_c:c\in\mathcal{C}_t\}\), while regions that do not match any known prototype with sufficient confidence are treated as unknown. Standard prototype learning optimizes these prototypes only through a recognition or detection objective, treating them as independent class anchors. Under scarce supervision, this leaves the prototype space underconstrained: many prototype configurations can explain the limited support examples, but induce different behavior for unknown rejection and novel-class insertion.

We address this limitation by introducing class-level geometry supervision. Let \(D\in\mathbb{R}^{C_t\times C_t}\) be a target class-dissimilarity matrix, where \(D_{ij}\) encodes the visual or semantic dissimilarity between classes \(i\) and \(j\). In our biomedical setting, we estimate visual geometry from training examples as
\begin{equation}
D^{\mathrm{vis}}_{ij}
=
\frac{1}{|S_i||S_j|}
\sum_{(I,b)\in S_i}
\sum_{(I',b')\in S_j}
\operatorname{MSE}\!\left(\phi(I,b),\phi(I',b')\right),
\end{equation}
where \(\phi(I,b)\) denotes the cropped and resized image region corresponding to box \(b\). This captures coarse pixel-level morphology differences between classes. We also consider language-derived geometry \(D^{\mathrm{txt}}_{ij}=1-\cos(e_i,e_j)\), where \(e_i\) and \(e_j\) are text embeddings of visual descriptions for classes \(i\) and \(j\). We use \(D\) as a nonnegative dissimilarity matrix, not necessarily a metric; the objective only requires pairwise relational targets. The matrix \(D\) is computed only from training/support information and provides a relational prior over the label space. 
The learned prototype space induces its own class-distance matrix \(G\), with entries \(G_{ij}=d(p_i,p_j)\), where \(d(\cdot,\cdot)\) is a distance in the prototype space, such as Euclidean or cosine distance. To compare distances that may have different scales, we normalize the off-diagonal entries as \(\hat{G}_{ij}=(G_{ij}-\mu_G)/\sigma_G\) and \(\hat{D}_{ij}=(D_{ij}-\mu_D)/\sigma_D\), where \(\mu_G,\sigma_G\) and \(\mu_D,\sigma_D\) are computed over class pairs \(i<j\). Let \(\mathcal{E}=\{(i,j):1\leq i<j\leq C_t\}\). We define the class-geometry loss as
\begin{equation}
\mathcal{L}_{\mathrm{CG}}
=
\frac{1}{|\mathcal{E}|}
\sum_{(i,j)\in\mathcal{E}}
\left(
\hat{G}_{ij}-\hat{D}_{ij}
\right)^2.
\end{equation}
The full objective is \(\mathcal{L}=\mathcal{L}_{\mathrm{task}}+\lambda\mathcal{L}_{\mathrm{CG}}\), where \(\mathcal{L}_{\mathrm{task}}\) denotes the standard prototype matching, classification, or detection loss, and \(\lambda\) controls class-geometry supervision strength.

\noindent\textbf{What Does Class-Geometry Supervision Control?}
The proposed objective aligns two complete weighted graphs over the class set: a target class graph with edge weights \(\hat{D}_{ij}\), encoding visual or semantic dissimilarity, and a learned prototype graph with edge weights \(\hat{G}_{ij}\), encoding distances between learned class prototypes. The following proposition shows that minimizing \(\mathcal{L}_{\mathrm{CG}}\) directly controls the average distortion between these graphs.

\begin{proposition}[Prototype-graph distortion bound]
If \(\mathcal{L}_{\mathrm{CG}}\leq\epsilon\), then the average absolute distortion between the learned prototype graph and the target class-dissimilarity graph satisfies
\[
\frac{1}{|\mathcal{E}|}
\sum_{(i,j)\in\mathcal{E}}
\left|
\hat{G}_{ij}-\hat{D}_{ij}
\right|
\leq
\sqrt{\epsilon}.
\]
\end{proposition}

Proposition~1 shows that \(\mathcal{L}_{\mathrm{CG}}\) is not an unconstrained regularizer: it explicitly minimizes average distortion between the learned prototype geometry and the target class geometry. While this controls average graph distortion, preserving local neighborhood topology requires sufficiently small pairwise distortion for the relevant class pairs. The next proposition formalizes this margin condition.

\begin{proposition}[Neighborhood-order preservation]
Suppose the target class geometry indicates that class \(j\) is closer to class \(i\) than class \(k\) by margin \(\gamma>0\), i.e., \(\hat{D}_{ij}+\gamma\leq\hat{D}_{ik}\). If the learned prototype distances satisfy \(|\hat{G}_{ij}-\hat{D}_{ij}|\leq\gamma/2\) and \(|\hat{G}_{ik}-\hat{D}_{ik}|\leq\gamma/2\), then the learned prototype space preserves this neighborhood relation:
$\hat{G}_{ij}\leq\hat{G}_{ik}$.
\end{proposition}

Proposition~2 gives a useful interpretation for scarce-data and incremental settings. If the target class geometry has a sufficient margin between nearby and distant classes, then a low-distortion prototype space preserves that neighborhood order. Thus, newly introduced classes can be placed into an existing prototype space according to their visual or semantic relationships with known classes, rather than being learned as isolated anchors. Proofs are provided in the supplementary.

\begin{algorithm}[t]
\caption{Class-Geometry Supervised Prototype Learning}
\label{alg:cg_prototype_learning}
\begin{algorithmic}[1]
\REQUIRE Support sets $\{S_{c}\}_{c=1}^{C_t}$, feature extractor $f_{\theta}$, task loss $\mathcal{L}_{\text{task}}$, geometry weight $\lambda$
\REQUIRE Class dissimilarity matrix $D \in \mathbb{R}^{C_t \times C_t}$
\ENSURE Geometry-supervised prototype set $\mathcal{P}_t = \{p_{c}\}_{c=1}^{C_t}$
\STATE Normalize off-diagonal entries of $D$ to obtain target geometry $\hat{D}$: $\hat{D}_{ij} \leftarrow (D_{ij} - \mu_D)/\sigma_D$
\FOR{each training episode or minibatch}
    \STATE Compute support class prototypes:
    \begin{equation*}
        p_{c} \leftarrow \frac{1}{|S_{c}|}\sum_{(I,b) \in S_{c}} f_{\theta}(I,b), \quad \forall c \in \{1, \dots, C_t\}
    \end{equation*}
    \STATE Form learned prototype distance matrix $G$ with entries $G_{ij} = d(p_{i}, p_{j})$
    \STATE Normalize off-diagonal entries of $G$ to obtain $\hat{G}$: $\hat{G}_{ij} \leftarrow (G_{ij} - \mu_G)/\sigma_G$
    \STATE Compute dissimilarity-preserving loss:
    \begin{equation*}
        \mathcal{L}_{\text{CG}} \leftarrow \frac{1}{|\mathcal{E}|}\sum_{(i,j)\in\mathcal{E}}(\hat{G}_{ij}-\hat{D}_{ij})^{2}
    \end{equation*}
    \STATE Update parameters by minimizing $\mathcal{L} \leftarrow \mathcal{L}_{\text{task}} + \lambda \mathcal{L}_{\text{CG}}$
\ENDFOR
\STATE Use $\mathcal{P}_t$ for nearest-prototype prediction, unknown rejection, or novel-class insertion
\RETURN $\mathcal{P}_t$
\end{algorithmic}
\end{algorithm}

\subsubsection{Open-World Decisions with Prototype Geometry.}

The propositions above show that \(\mathcal{L}_{\mathrm{CG}}\) controls the distortion between the learned prototype graph and the target class-dissimilarity graph. We now describe how this geometry-supervised prototype graph is used for open-world decisions. 
Algorithm~\ref{alg:cg_prototype_learning} illustrates the process. 
Given a query region feature \(z=f_\theta(I,b)\), we compute its distance to each known-class prototype \(p_c\in\mathcal{P}_t\). The predicted known label is
\[
\hat{y}
=
\arg\min_{c\in\mathcal{C}_t}
d(z,p_c).
\]
In our implementation, the open-world score is the nearest-prototype distance: 
$m(z){=}\min_{c\in\mathcal{C}_t} d(z,p_c)$.
A region is assigned to its nearest known class only if this distance is below a global validation-selected rejection threshold \(\tau_{\mathrm{unk}}\); otherwise it is marked as unknown:
\[
\mathrm{unknown}(z)
=
\mathbb{I}
\left[
m(z) > \tau_{\mathrm{unk}}
\right].
\]
Thus, unknown rejection depends on the calibration of distances between query features and the known prototype space. 
Class-geometry supervision improves this decision rule by organizing known prototypes according to visual or semantic relationships. If prototypes are learned as independent anchors, nearest-prototype distances may be poorly calibrated: visually similar classes may be separated arbitrarily, while dissimilar classes may collapse. In contrast, minimizing \(\mathcal{L}_{\mathrm{CG}}\) reduces distortion between the learned prototype graph and the target class geometry, making the known-class manifold more structured. This is especially important for open-world detection, where unknown regions are identified by their failure to fit the known prototype geometry.

The same structure supports novel-class insertion. When a set of newly labeled classes \(\Delta\mathcal{C}\) becomes available at stage \(t+1\), we compute support-derived prototypes \(p_c\) for \(c\in\Delta\mathcal{C}\) and expand the prototype set as
$\mathcal{P}_{t+1} {=} \mathcal{P}_t \cup \{p_c:c\in\Delta\mathcal{C}\}$.
In the incremental setting, existing prototypes in \(\mathcal{P}_t\) are kept fixed unless explicitly fine-tuned, and only the newly introduced class prototypes are added from their limited support examples. The class-dissimilarity matrix is similarly expanded to include relationships between new and existing classes. The geometry objective then encourages new prototypes to be inserted according to their visual or semantic relationships with the previous class set, rather than being learned as isolated anchors. This connects directly to Proposition~2: when the target geometry has a sufficient margin between nearby and distant classes, low pairwise distortion preserves the intended neighborhood order after insertion.

\begin{figure*}
    \includegraphics[width=\linewidth]{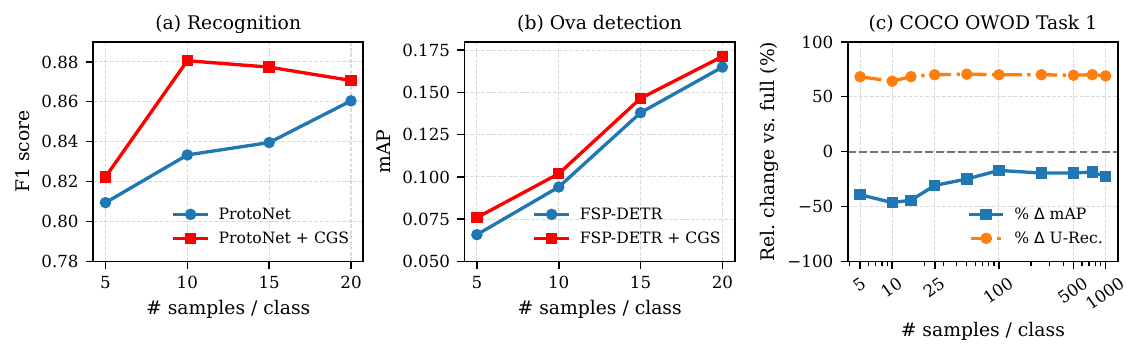}
    \caption{
\textbf{Sample efficiency across recognition and detection.}
CGS improves performance across three increasingly challenging settings.
(a) In prototype recognition, CGS improves low-shot classification.
(b) In ova detection, CGS consistently improves FSP-DETR mAP.
(c) On COCO Task~1, we freeze detector features and fine-tune only the classification MLP head with \(K\) labeled examples per class; we report relative change with respect to the Task 1 full OW-DETR reference. 
}
\label{fig:sample_efficiency}
\end{figure*}

\subsubsection{Applying Class-Geometry Supervision.}

The formulation above only assumes that the model produces class prototypes or class-specific embeddings. We therefore apply the same class-geometry objective in three settings: prototype recognition, prototype-based biomedical detection, and open-world detector adaptation. 

\noindent\textbf{Prototype recognition.}
We first apply class-geometry supervision to ProtoNet-style recognition to isolate its effect on scarce-data prototype learning. Given support examples for each class, prototypes are computed as class centroids in the embedding space and query images are classified by nearest-prototype matching. In this setting, \(\mathcal{L}_{\mathrm{task}}\) is the standard prototype matching loss, and \(\mathcal{L}_{\mathrm{CG}}\) shapes only the inter-class prototype geometry. This experiment evaluates whether preserving visual or semantic class relationships improves recognition before introducing detection-specific factors.

\noindent\textbf{Prototype-based biomedical detection.}
We next apply the same objective to scarce-data biomedical detection. A DETR-style detector produces object queries and region features, and class labels are assigned by matching each region feature to support-derived class prototypes. The objective combines localization, matching, and CGS supervision:
\[
\mathcal{L}
=
\mathcal{L}_{\mathrm{det}}
+
\mathcal{L}_{\mathrm{proto}}
+
\lambda\mathcal{L}_{\mathrm{CG}}.
\]
Here, \(\mathcal{L}_{\mathrm{CG}}\) does not alter the localization objective; it structures the class prototype space used for known-class recognition, unknown rejection, and novel-class insertion. We use this instantiation for few-shot ova detection, open-set detection, and incremental 15-to-20 class learning.

\noindent\textbf{Open-world detector adaptation.}
Finally, to test whether class-geometry supervision transfers beyond the biomedical benchmark, we apply it to the OW-DETR/COCO protocol. We evaluate two adaptation regimes. In the frozen-head setting, OW-DETR detector features are fixed and only a lightweight two-layer classification head is fine-tuned with the CGS objective applied to the induced class representations. In the full fine-tuning setting, the detector is optimized end-to-end with the same class-geometry loss added to the standard OW-DETR training objective. These settings allow us to test whether relational class supervision can improve unknown discovery at the classifier level alone, and whether full model adaptation can reduce the known-class mAP tradeoff while retaining the open-world benefit.

\noindent\textbf{Implementation Details.}
All models use the same image preprocessing, support-set construction, and train/test splits as the corresponding baseline unless otherwise stated. For recognition, we train a ProtoNet backbone with the standard episodic matching loss and add \(\mathcal{L}_{\mathrm{CG}}\) over episode prototypes using \(\lambda\in\{\mathrm{0.05, 0.01, 0.1, 0.5, 1}\}\). For biomedical detection, we initialize a DETR model and apply the geometry loss after training using the vanilla DETR loss, using support-derived prototypes from \(K\in\{5,10,15,20\}\) examples per class. Visual dissimilarities are computed from cropped and resized support regions at resolution \(224\times 224\); text dissimilarities use MPNet \cite{Songetal2020} embeddings of Google Gemini \cite{geminiteam2024gemini15} generated visual descriptions. 
For OW-DETR/COCO, we evaluate CGS in two regimes: frozen-head adaptation, where OW-DETR features are fixed and a two-layer classification MLP is fine-tuned, and full fine-tuning.
Thresholds for unknown rejection, stress weights, and normalization hyperparameters are selected on the validation split; additional details in supplementary. 

\section{Experimental Evaluation}

\begin{table}[t]
\centering
\setlength{\tabcolsep}{2pt}
\resizebox{\linewidth}{!}{
\begin{tabular}{lcc|cc}
\toprule
\multirow{2}{*}{\textbf{Method}} 
& \multicolumn{2}{c|}{\textbf{Few-shot Detection mAP}} 
& \multicolumn{2}{c}{\textbf{Open-set Detection}} \\
\cmidrule(lr){2-3}
\cmidrule(lr){4-5}
& \textbf{5-shot} & \textbf{10-shot} 
& \textbf{mAP} & \textbf{mAR} \\
\midrule
\multicolumn{5}{c}{\textit{Generic detectors}} \\
\midrule
DETR
& 0.026 & 0.054 & N/A & N/A- \\
YOLO
& 0.000 & 0.021 & N/A & N/A \\
Stable-DINO
& 0.002 & 0.011 & N/A & N/A \\
FCOS
& 0.001 & 0.016 & N/A & N/A \\

\midrule
\multicolumn{5}{c}{\textit{Prototype / few-shot baselines}} \\
\midrule
ProtoNet + SS
& 0.036 & 0.045 & 0.003 & 0.026 \\
ProtoKD + SS
& 0.039 & 0.048 & 0.031 & 0.109 \\
FSP-DETR
& 0.066 & 0.094 & 0.045 & \textbf{0.141} \\

\midrule
\rowcolor{gray!8}
FSP-DETR + CGS 
& \textbf{0.076 \textcolor{green!60!black}{\(\uparrow\)}} 
& \textbf{0.102 \textcolor{green!60!black}{\(\uparrow\)}} 
& \textbf{0.061 \textcolor{green!60!black}{\(\uparrow\)}} 
& 0.136 \textbf{ \textcolor{red!70!black}{\(\downarrow\)}} \\
\bottomrule
\end{tabular}
}
\caption{
\textbf{Ova detection under scarce-data and open-set settings.}
Arrows indicate improvement/decline relative to FSP-DETR without CGS.
N/A: model not tailored for task.
}
\label{tab:ova_detection_main}
\end{table} 

\textbf{Tasks and Datasets.} We evaluate class-geometry supervision across three complementary settings designed to test whether relational class structure improves sample-efficient open-world learning. First, we use a parasitic ova benchmark~\cite{trehan2026fsp} for low-shot prototype recognition and few-shot detection, where fine-grained categories must be learned from limited support crops or bounding boxes. Second, on the same ova dataset, we evaluate open-set detection and novel-class insertion, testing whether geometry-supervised prototypes improve unknown rejection and the integration of newly labeled classes. Third, we evaluate on the standard OW-DETR/COCO \cite{Gupta2022} open-world detection protocol to test whether the same idea transfers beyond biomedical imagery, including both frozen-head adaptation and full-model fine-tuning regimes. Across these settings, we report recognition F1/accuracy, detection mAP and mAR, unknown recall, and seen/novel-class performance.

\noindent\textbf{Metrics and Baselines.}
We use task-standard metrics for each evaluation setting and compare against baselines matched to each task. For prototype-based recognition, we report macro-F1 over query examples and compare against a standard ProtoNet~\cite{Snell2017} trained with the same episodic protocol but without class-geometry supervision. For ova detection, we report mean Average Precision (mAP) and mean Average Recall (mAR), comparing against generic detectors including DETR~\cite{Carion2020}, YOLO~\cite{jocher2020ultralytics}, Stable-DINO~\cite{liu2023detection}, and FCOS~\cite{tian2019fcos}, as well as prototype/few-shot baselines such as ProtoNet+SS~\cite{Snell2017} and ProtoKD+SS~\cite{trehan2023protokd}. Our primary controlled comparison is FSP-DETR~\cite{trehan2026fsp}, the strongest prototype-based detector baseline, with and without CGS. For open-set detection, mAP and mAR are computed with the unknown category included, and for novel-class insertion we report mAP and mAR under \textit{novel-only}, \textit{generalized novel}, and \textit{seen retention} settings using the same known/novel splits. For COCO open-world detection, we follow the standard OWOD protocol~\cite{Joseph2021}, reporting unknown recall (U-Recall) and mAP over previously known, currently introduced, and all known classes, and compare against ORE~\cite{Joseph2021}, UC-OWOD~\cite{Wu2022}, OW-DETR~\cite{Gupta2022}, Fast-OWDETR~\cite{chen2022fastowdetr}, OCPL~\cite{Yu2023}, and RE-OWOD~\cite{Zhao2024}. 

\begin{table}[t]
\centering
\resizebox{\linewidth}{!}{
\begin{tabular}{llcc}
\toprule
\textbf{Setting} & \textbf{Method} & \textbf{mAP} & \textbf{mAR} \\
\midrule
\multirow{4}{*}{Novel-only}
& ProtoNet + SS & 0.0486 & 0.2146 \\
& ProtoKD + SS & 0.0532 & 0.2134 \\
& FSP-DETR & 0.0629 & 0.2595 \\
\rowcolor{gray!8}
& FSP-DETR + CGS & \textbf{0.1558 \textcolor{green!60!black}{\(\uparrow\)}} & \textbf{0.3728 \textcolor{green!60!black}{\(\uparrow\)}} \\
\midrule
\multirow{4}{*}{Generalized novel}
& ProtoNet + SS & 0.0340 & 0.1710 \\
& ProtoKD + SS & 0.0520 & 0.1960 \\
& FSP-DETR & 0.0569 & 0.2138 \\
\rowcolor{gray!8}
& FSP-DETR + CGS & \textbf{0.1348 \textcolor{green!60!black}{\(\uparrow\)}} & \textbf{0.3248 \textcolor{green!60!black}{\(\uparrow\)}} \\
\midrule
\multirow{4}{*}{Seen retention}
& ProtoNet + SS & 0.0033 & 0.0236 \\
& ProtoKD + SS & 0.0380 & 0.1107 \\
& FSP-DETR & 0.0566 & 0.1173 \\
\rowcolor{gray!8}
& FSP-DETR + CGS & \textbf{0.0670 \textcolor{green!60!black}{\(\uparrow\)}} & \textbf{0.1259 \textcolor{green!60!black}{\(\uparrow\)}} \\
\bottomrule
\end{tabular}
}
\caption{
\textbf{Novel-class insertion.}
CGS yields the largest gains when new classes are inserted into the prototype space.
}
\label{tab:novel_insertion_cgs}
\end{table}

\noindent\textbf{Sample Efficiency Across Tasks.} 
\label{sec:sample_efficiency} 
We first evaluate whether class-geometry supervision improves sample efficiency across recognition, detection, and open-world adaptation. Figure~\ref{fig:sample_efficiency} shows that CGS improves ProtoNet recognition and FSP-DETR ova detection across support sizes, with the largest gains in the lowest-shot regimes. This indicates that the benefit is not detection-specific: when class prototypes are estimated from few examples, preserving inter-class geometry provides additional supervision that helps organize the class space before more data are available. In ova detection, this translates into consistent mAP gains over FSP-DETR without changing the detector backbone, suggesting that CGS improves the use of scarce support examples rather than merely increasing model capacity. On COCO Task~1, we freeze OW-DETR features and fine-tune only the classification head with \(K\) labeled examples per class; Fig.~\ref{fig:sample_efficiency}(c) reports relative change with respect to the full OW-DETR reference. CGS consistently increases unknown recall relative to the full model, although with lower known-class mAP. Thus, CGS shifts the adaptation toward open-world sensitivity: it is most useful when the goal is not only to maximize known-class accuracy, but to improve the model's ability to detect objects outside the current label set under limited supervision.

\begin{table}[t]
\centering
\small
\setlength{\tabcolsep}{5pt}
\begin{tabular}{lccc}
\toprule
\textbf{Variant} & \textbf{5-shot} & \textbf{10-shot} & \textbf{Novel-only} \\
\midrule
No CGS 
& 0.0658 & 0.0940 & 0.0629 \\
Random \(D\) 
& 0.0710 & 0.0975 & \textbf{0.1657} \\
Text \(D_{\mathrm{txt}}\) 
& 0.0717 & 0.0989 & 0.1476 \\
Visual + Text 
& 0.0741 & 0.0984 & 0.1581 \\
\rowcolor{gray!8}
Visual \(D_{\mathrm{vis}}\) 
& \textbf{0.0759} & \textbf{0.1019} & {0.1558} \\
\bottomrule
\end{tabular}
\caption{
\textbf{Effect of class-geometry source on ova detection.}
We compare different target class geometries used by CGS. 
}
\label{tab:ablation_geometry_source}
\end{table}

\begin{table*}[t]
\centering
\small
\setlength{\tabcolsep}{3.2pt}
\resizebox{\textwidth}{!}{
\begin{tabular}{lcc|cc|cccc|cccc|ccc}
\toprule
\multirow{3}{*}{\textbf{Method}}
& \multirow{3}{*}{\textbf{Trainable}}
& \multirow{3}{*}{\textbf{Data}}
& \multicolumn{2}{c|}{\textbf{Task 1}}
& \multicolumn{4}{c|}{\textbf{Task 2}}
& \multicolumn{4}{c|}{\textbf{Task 3}}
& \multicolumn{3}{c}{\textbf{Task 4}} \\
\cmidrule(lr){4-5}
\cmidrule(lr){6-9}
\cmidrule(lr){10-13}
\cmidrule(lr){14-16}
&
&
& \multirow{2}{*}{\textbf{U-Rec.}}
& \textbf{mAP}
& \multirow{2}{*}{\textbf{U-Rec.}}
& \multicolumn{3}{c|}{\textbf{mAP}}
& \multirow{2}{*}{\textbf{U-Rec.}}
& \multicolumn{3}{c|}{\textbf{mAP}}
& \multicolumn{3}{c}{\textbf{mAP}} \\
\cmidrule(lr){5-5}
\cmidrule(lr){7-9}
\cmidrule(lr){11-13}
\cmidrule(lr){14-16}
&
&
& 
& \textbf{Curr.}
&
& \textbf{Prev.} & \textbf{Curr.} & \textbf{Both}
&
& \textbf{Prev.} & \textbf{Curr.} & \textbf{Both}
& \textbf{Prev.} & \textbf{Curr.} & \textbf{Both} \\
\midrule
ORE~\cite{Joseph2021}
& full & full
& 4.9 & 56.0
& 2.9 & 52.7 & 26.0 & 39.4
& 3.9 & 38.2 & 12.7 & 29.7
& 29.6 & 12.4 & 25.3 \\

UC-OWOD~\cite{Wu2022}
& full & full
& 75.4 & 56.4
& 68.0 & 33.1 & 30.5 & 46.1
& 53.3 & 28.8 & 16.3 & 28.0
& 25.6 & 15.9 & 26.7 \\

Fast-OWDETR~\cite{chen2022fastowdetr}
& full & full
& 9.2 & 56.6
& 8.8 & -- & -- & 39.4
& 7.8 & -- & -- & 32.2
& -- & -- & 25.0 \\

OCPL~\cite{Yu2023}
& full & full
& 8.3 & 56.6
& 7.7 & 50.7 & 27.5 & 39.1
& {11.9} & 38.6 & 14.7 & 30.7
& 30.8 & 14.4 & 26.7 \\

RE-OWOD~\cite{Zhao2024}
& full & full
& 17.9 & {61.3}
& {11.8} & 55.6 & 37.9 & {46.6}
& 9.9 & 46.5 & 31.3 & {38.9}
& 38.9 & 19.8 & {33.9} \\

OW-DETR~\cite{Gupta2022}
& full & full
& 7.5 & 59.2
& 6.2 & 53.6 & 33.5 & 42.9
& 5.7 & 38.3 & 15.8 & 30.8
& 31.4 & 17.1 & 27.8 \\

\midrule
\rowcolor{gray!8}
Frozen OW-DETR + CGS
& head & full
& \textbf{11.9}\textbf{\textcolor{green!60!black}{\(\uparrow\)}} & \textbf{62.2}\textbf{\textcolor{green!60!black}{\(\uparrow\)}}
& \textbf{9.1}\textbf{\textcolor{green!60!black}{\(\uparrow\)}} & \textbf{53.3}\textbf{\textcolor{red!70!black}{\(\downarrow\)}} & \textbf{19.4} \textbf{\textcolor{red!70!black}{\(\downarrow\)}}& \textbf{35.6}\textbf{\textcolor{red!70!black}{\(\downarrow\)}}
& \textbf{9.8}\textbf{\textcolor{green!60!black}{\(\uparrow\)}} & \textbf{37.4}\textbf{\textcolor{red!70!black}{\(\downarrow\)}} & \textbf{16.8}\textbf{\textcolor{green!60!black}{\(\uparrow\)}}  & \textbf{30.5}\textbf{\textcolor{red!70!black}{\(\downarrow\)}}
& \textbf{30.7}\textbf{\textcolor{red!70!black}{\(\downarrow\)}} & \textbf{21.4}\textbf{\textcolor{green!60!black}{\(\uparrow\)}} & \textbf{28.4}\textbf{\textcolor{green!60!black}{\(\uparrow\)}} \\

\rowcolor{gray!8}
OW-DETR + CGS
& full & full
& \textbf{7.5}\textbf{\textcolor{green!60!black}{\(\uparrow\)}} & \textbf{59.5}\textbf{\textcolor{green!60!black}{\(\uparrow\)}}
& \textbf{9.8}\textbf{\textcolor{green!60!black}{\(\uparrow\)}} & \textbf{54.2}\textbf{\textcolor{green!60!black}{\(\uparrow\)}} & \textbf{28.2}\textbf{\textcolor{red!70!black}{\(\downarrow\)}} & \textbf{40.2}\textbf{\textcolor{red!70!black}{\(\downarrow\)}}
& \textbf{8.2}\textbf{\textcolor{green!60!black}{\(\uparrow\)}} & \textbf{41.5}\textbf{\textcolor{green!60!black}{\(\uparrow\)}} & \textbf{21.2}\textbf{\textcolor{green!60!black}{\(\uparrow\)}} & \textbf{34.7}\textbf{\textcolor{green!60!black}{\(\uparrow\)}}
& \textbf{34.5}\textbf{\textcolor{green!60!black}{\(\uparrow\)}} & \textbf{24.6}\textbf{\textcolor{green!60!black}{\(\uparrow\)}} & \textbf{32.0}\textbf{\textcolor{green!60!black}{\(\uparrow\)}} \\

\bottomrule
\end{tabular}
}
\caption{
\textbf{OWOD evaluation on COCO}.
We report U-Rec. and mAP over previously known, currently known, and all known classes.
Arrows show improvement or decline relative to OW-DETR. CGS results are bolded.
}
\label{tab:coco_owod_sota}
\end{table*}

\noindent\textbf{Scarce-Data and Open-Set Ova Detection.}
\label{sec:ova_detection_results} 
Table~\ref{tab:ova_detection_main} compares CGS with generic detectors, prototype/few-shot baselines, and the strongest prototype detector baseline on ova detection. In the few-shot setting, generic detectors perform poorly under limited supervision, and prototype-based methods provide stronger performance. Among these, FSP-DETR is the strongest baseline, achieving 0.066 mAP in the 5-shot setting and 0.094 mAP in the 10-shot setting. Adding class-geometry supervision improves FSP-DETR to 0.076 and 0.102 mAP, respectively. These gains indicate that CGS improves the class prototype space without changing the detector backbone, with the largest relative improvement occurring in the most data-scarce 5-shot regime. 
We also evaluate open-set ova detection, where the model must reject objects outside the known class set. CGS improves open-set mAP from 0.045 to 0.061 over FSP-DETR, while maintaining comparable recall. This suggests that organizing the known-class prototype manifold with visual class geometry improves the nearest-prototype distance used for unknown rejection. The slight decrease in mAR reflects a precision--recall tradeoff: CGS is more selective in assigning detections as ``unknown'', resulting in higher mAP at lower recall.

\noindent\textbf{Novel-Class Insertion.} 
\label{sec:novel_insertion_results} 
Table~\ref{tab:novel_insertion_cgs} evaluates whether class-geometry supervision helps insert newly labeled classes into an existing prototype space. We consider three settings: \textit{novel-only}, where only the newly introduced classes are evaluated; \textit{generalized novel}, where all prototypes are available but performance is measured on novel classes; and \textit{seen retention}, where performance is measured on previously known classes after insertion. CGS provides substantial gains in the novel-only setting, improving mAP from 0.0629 to 0.1558 and mAR from 0.2595 to 0.3728. The gains remain large in the generalized novel setting, where mAP improves from 0.0569 to 0.1348 and mAR improves from 0.2138 to 0.3248. These results indicate that CGS helps newly introduced prototypes occupy meaningful positions in the existing class manifold rather than being learned as isolated anchors. 
CGS also improves seen-class retention, increasing mAP from 0.0566 to 0.0670 and mAR from 0.1173 to 0.1259. This suggests that the benefit of class geometry is not limited to novel classes; organizing the prototype space also helps preserve discrimination among previously known categories.

\noindent\textbf{OWOD Adaptation on COCO.}
\label{sec:coco_owod_results}
Table~\ref{tab:coco_owod_sota} evaluates whether class-geometry supervision transfers beyond the ova benchmark to the standard OWOD protocol on COCO. Rather than treating COCO only as a closed-set mAP benchmark, we use it to measure the tradeoff between known-class detection and unknown discovery. We compare OW-DETR with two CGS variants: a frozen-head setting, where only the classification head is updated over fixed OW-DETR features, and a full fine-tuning setting, where the entire model is optimized with the CGS objective. 
The results show that CGS consistently shifts the detector toward stronger open-world sensitivity. In the frozen-head setting, CGS substantially improves U-Recall over OW-DETR across Tasks~1--3, showing that reorganizing the class space at the head level is sufficient to make the detector more responsive to objects outside the known label set. This comes with some loss in known-class mAP, especially on later tasks, which is expected because the backbone and localization components are fixed while the head is adapted toward a more geometry-aware class space. 
Full-model CGS fine-tuning provides a lower-tradeoff operating point. Compared with head-only adaptation, it preserves more of the known-class mAP while still improving U-Recall on later tasks. This suggests that when the full detector is allowed to adapt, the model can partially recover known-class discrimination while retaining the open-world benefit of geometry supervision. CGS does not simply maximize mAP, but offers a transferable way to improve unknown discovery with a controllable mAP--U-Recall tradeoff. Figure~\ref{fig:coco_qual_main} illustrates examples to highlight this tradeoff qualitatively.

\noindent\textbf{Source of Class Geometry.}
\label{sec:ablation_geometry_source}
Table~\ref{tab:ablation_geometry_source} ablates the target geometry used by CGS. All geometry variants improve over the no-CGS baseline, showing that the pairwise relational objective helps structure the prototype space under scarce supervision. Meaningful geometry is more reliable for few-shot detection: visual geometry \(D_{\mathrm{vis}}\) gives the best 5-shot and 10-shot mAP, while text-derived geometry and visual+text geometry also improve over the baseline but do not surpass \(D_{\mathrm{vis}}\). This suggests that support-crop morphology provides the strongest relational signal for fine-grained ova detection. 
The novel-only setting shows a different pattern: random \(D\) (with permutated class-distance matrix \(D_{\mathrm{vis}}\)) performs strongly, suggesting that even arbitrary pairwise spreading can help separate newly inserted prototypes from existing classes. However, random geometry is less consistent and does not encode class relationships. 

\begin{figure}[t]
    \begin{tabular}{cc}
        \includegraphics[width=0.95\linewidth]{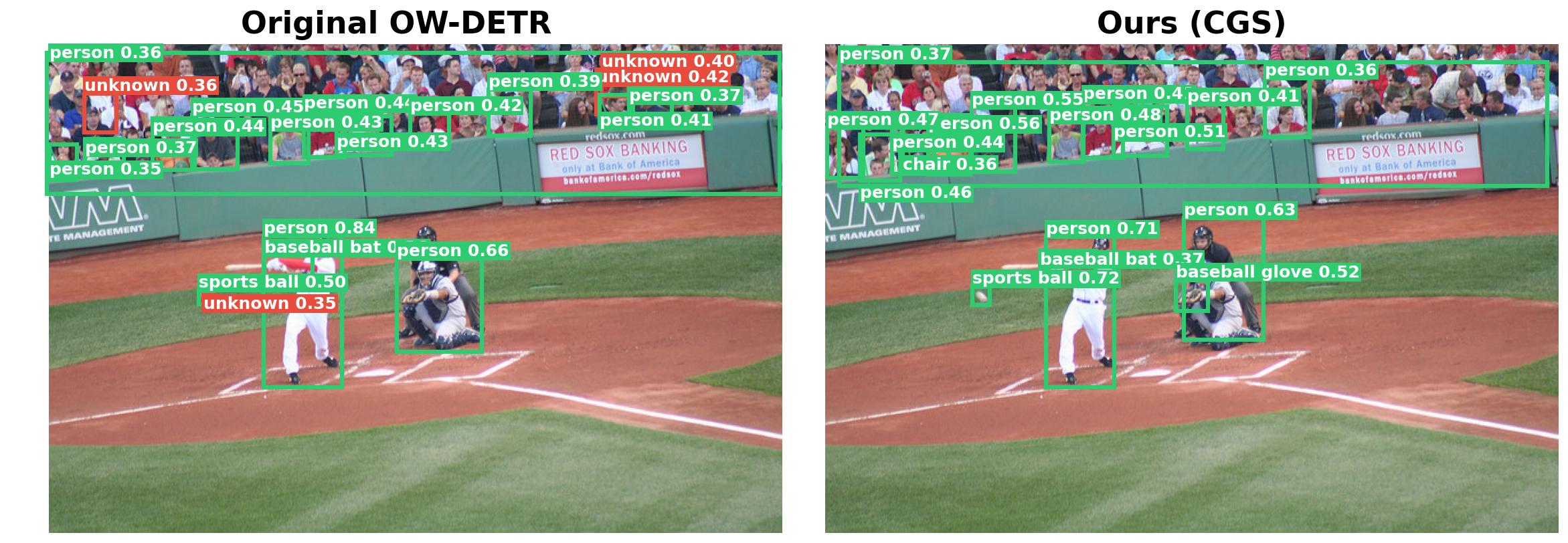}\\
        \small{\textbf{Success:} More known objects recovered (ball, person, glove)}.\\
        \includegraphics[width=0.95\linewidth]{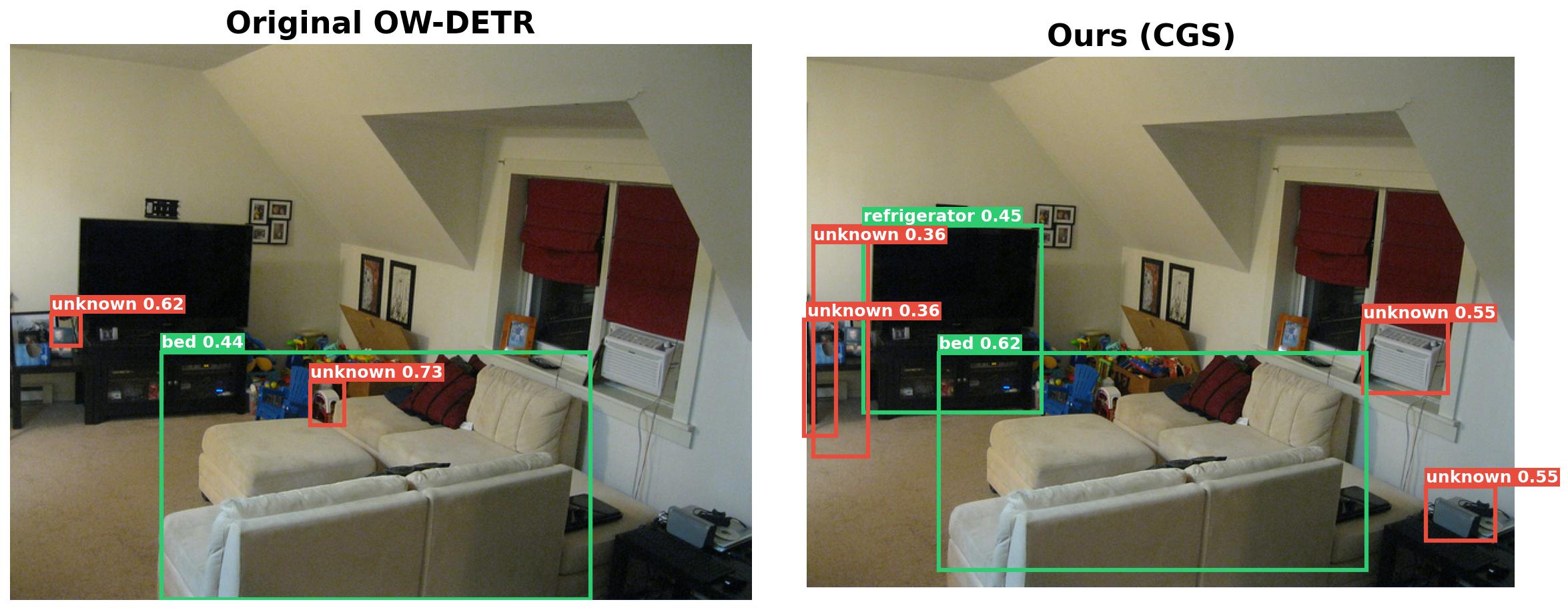}\\
        \small{\textbf{Failure:} semantic confusion and extra unknown detections.}\\
    \end{tabular}
    \caption{
    \textbf{Qualitative comparison on OWOD COCO.}
    Each row compares OW-DETR (left) and OW-DETR+CGS (right). 
    Green boxes: known-class detections; red: unknown. 
    }
    \label{fig:coco_qual_main}
\end{figure}

\section{Conclusion and Future Work}
We introduced class-geometry supervision, a simple relational objective for organizing prototype and class-representation spaces in sample-efficient open-world detection. Across recognition, ova detection, open-set detection, novel-class insertion, and COCO OWOD adaptation, CGS improves prototype calibration and strengthens open-world behavior without requiring a new detector architecture. The results show that class geometry is especially useful for novel insertion and unknown recall, though it can introduce a known-class mAP tradeoff, particularly in frozen-head adaptation. Our ablations further show that the choice of target geometry matters: visual morphology provides the most reliable gains, while random geometry can help separation but lacks consistent and interpretable structure. Future work will address the broader challenge of how to learn or update class geometry reliably as new categories arrive, especially when visual and semantic geometry disagree.

\textbf{Acknowledgement.} This work was supported in part by the US NSF grants IIS 2348689, IIS 2348690, IIS 2615771, and the USDA award no. 2023-69014-39716. 
\bibliography{aaai2027}

\clearpage

\section{Appendix 1}

\section{Extended Mathematical Analysis}

\subsection{What Does Class-Geometry Supervision Control?}

The proposed objective can be interpreted as aligning two complete weighted graphs over the class set. The target graph has edge weights \(\hat{D}_{ij}\), encoding visual or semantic class dissimilarity, while the learned prototype graph has edge weights \(\hat{G}_{ij}\), encoding distances between learned class prototypes. The following proposition shows that minimizing \(\mathcal{L}_{\mathrm{CG}}\) directly controls the average distortion between these two graphs.

\begin{proposition}[Prototype-graph distortion bound]
Let
\[
\mathcal{L}_{\mathrm{CG}}
=
\frac{1}{|\mathcal{E}|}
\sum_{i<j}
\left(
\hat{G}_{ij}-\hat{D}_{ij}
\right)^2
\]
be the class-geometry loss over the class-pair set \(\mathcal{E}\). If
\[
\mathcal{L}_{\mathrm{CG}} \leq \epsilon,
\]
then the average absolute distortion between the learned prototype graph and the target class-dissimilarity graph is bounded as
\[
\frac{1}{|\mathcal{E}|}
\sum_{i<j}
\left|
\hat{G}_{ij}-\hat{D}_{ij}
\right|
\leq
\sqrt{\epsilon}.
\]
\end{proposition}

\begin{proof}
By Cauchy--Schwarz,
\[
\frac{1}{|\mathcal{E}|}
\sum_{i<j}
\left|
\hat{G}_{ij}-\hat{D}_{ij}
\right|
\leq
\sqrt{
\frac{1}{|\mathcal{E}|}
\sum_{i<j}
\left(
\hat{G}_{ij}-\hat{D}_{ij}
\right)^2
}.
\]
The term inside the square root is exactly \(\mathcal{L}_{\mathrm{CG}}\). Therefore,
\[
\frac{1}{|\mathcal{E}|}
\sum_{i<j}
\left|
\hat{G}_{ij}-\hat{D}_{ij}
\right|
\leq
\sqrt{\mathcal{L}_{\mathrm{CG}}}
\leq
\sqrt{\epsilon}.
\]
\end{proof}

Thus, \(\mathcal{L}_{\mathrm{CG}}\) explicitly minimizes average distortion between the learned prototype graph and the target class-dissimilarity graph. This makes the objective a relational graph-alignment criterion rather than an unconstrained regularizer. While Proposition~1 controls average distortion, preserving local neighborhood topology requires sufficiently small pairwise distortion for the relevant class pairs. The next proposition formalizes this margin condition.

\begin{proposition}[Neighborhood-order preservation]
Suppose the target class geometry indicates that class \(j\) is closer to class \(i\) than class \(k\) by margin \(\gamma>0\):
\[
\hat{D}_{ij} + \gamma \leq \hat{D}_{ik}.
\]
If the learned prototype distances satisfy
\[
\left|
\hat{G}_{ab}-\hat{D}_{ab}
\right|
\leq
\frac{\gamma}{2}
\qquad
\text{for all relevant class pairs } (a,b),
\]
then the learned prototype space preserves this neighborhood relation:
\[
\hat{G}_{ij}\leq \hat{G}_{ik}.
\]
\end{proposition}

\begin{proof}
Using the assumed distortion bound,
\[
\hat{G}_{ij}
\leq
\hat{D}_{ij}+\frac{\gamma}{2}.
\]
Since \(\hat{D}_{ij}+\gamma \leq \hat{D}_{ik}\), we have
\[
\hat{D}_{ij}+\frac{\gamma}{2}
\leq
\hat{D}_{ik}-\frac{\gamma}{2}.
\]
Again using the distortion bound,
\[
\hat{D}_{ik}-\frac{\gamma}{2}
\leq
\hat{G}_{ik}.
\]
Combining these inequalities gives
\[
\hat{G}_{ij}
\leq
\hat{G}_{ik},
\]
which proves the claim.
\end{proof}

Proposition~2 shows that when the target class geometry has a sufficient margin between nearby and distant classes, low pairwise distortion preserves the corresponding neighborhood order in the learned prototype space. In practice, minimizing \(\mathcal{L}_{\mathrm{CG}}\) reduces average graph distortion, while empirical analyses of distance correlations and nearest-neighbor consistency can verify whether this relational structure is preserved for the learned prototypes.

\begin{figure*}[t]
    \centering
    \begin{tabular}{cc}
            \includegraphics[width=0.4\linewidth, cfbox=green 2pt]{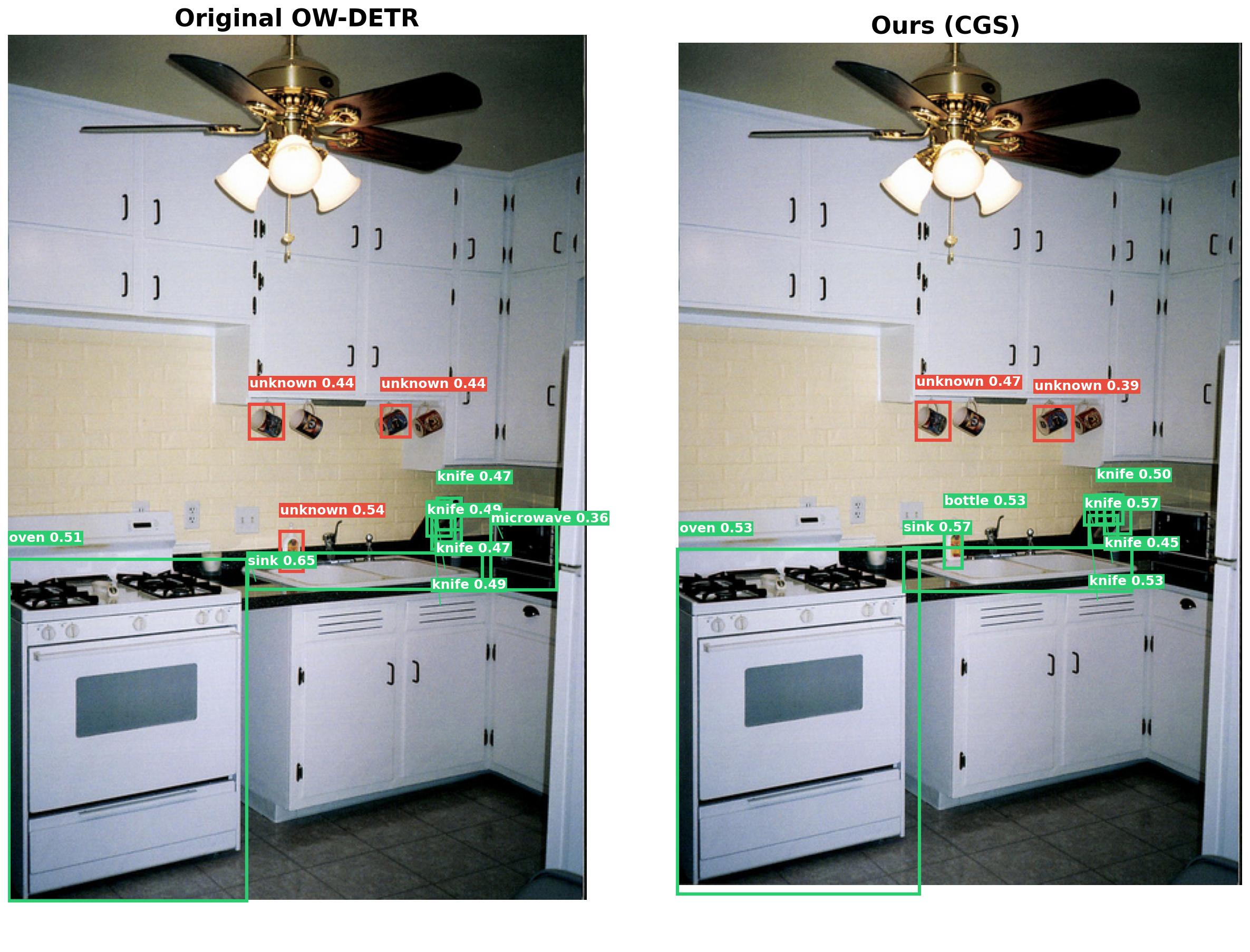} & 
            \includegraphics[width=0.4\linewidth,cfbox=green 2pt]{ 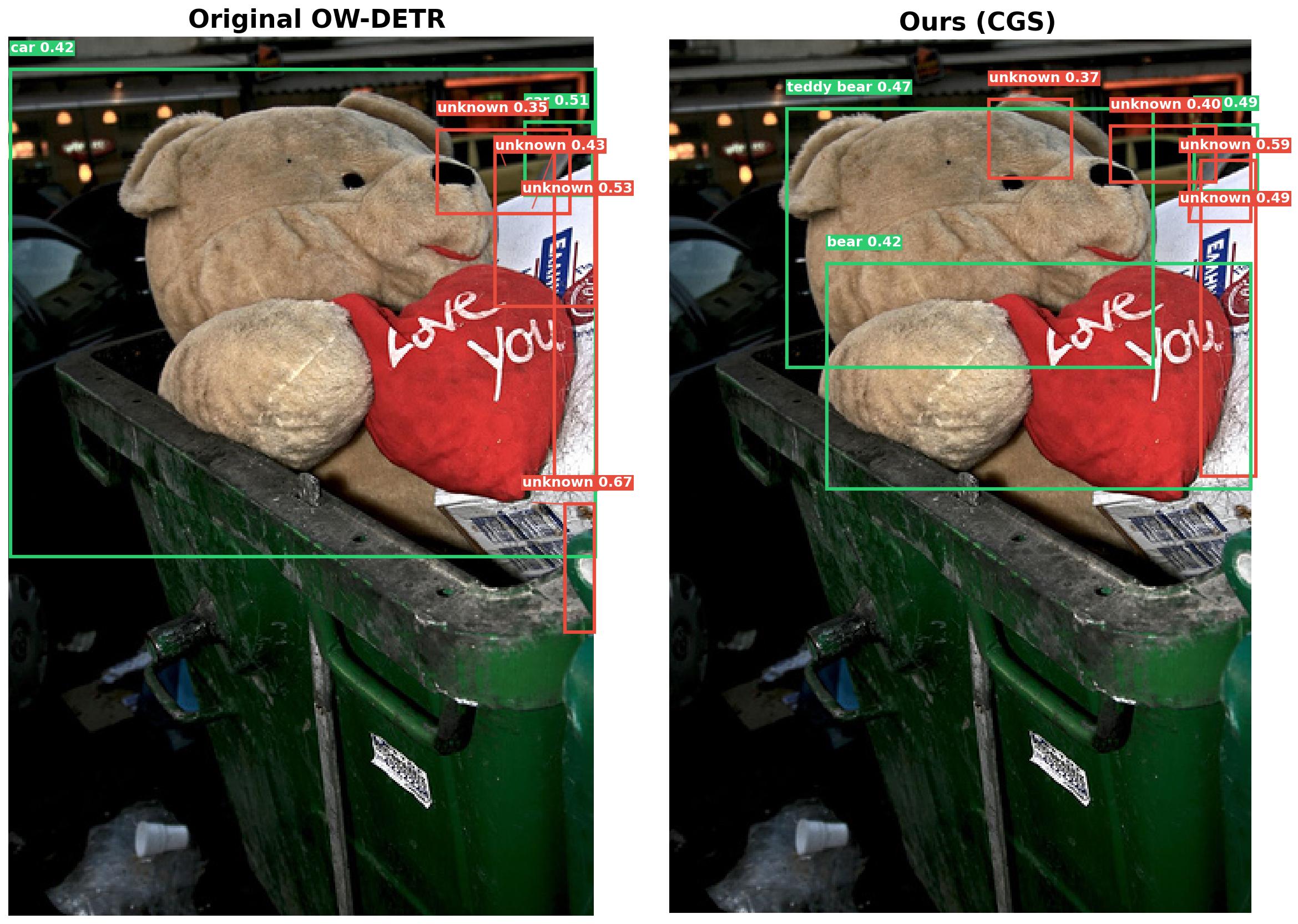} \\ 
            Bottle recovered from unknown. & 
            Bear corrected from a wrong known label.\\
            \includegraphics[width=0.4\linewidth,cfbox=green 2pt]{ Figures/qualitative_000000060507.jpg} & 
            \includegraphics[width=0.4\linewidth, cfbox=red 2pt]{ 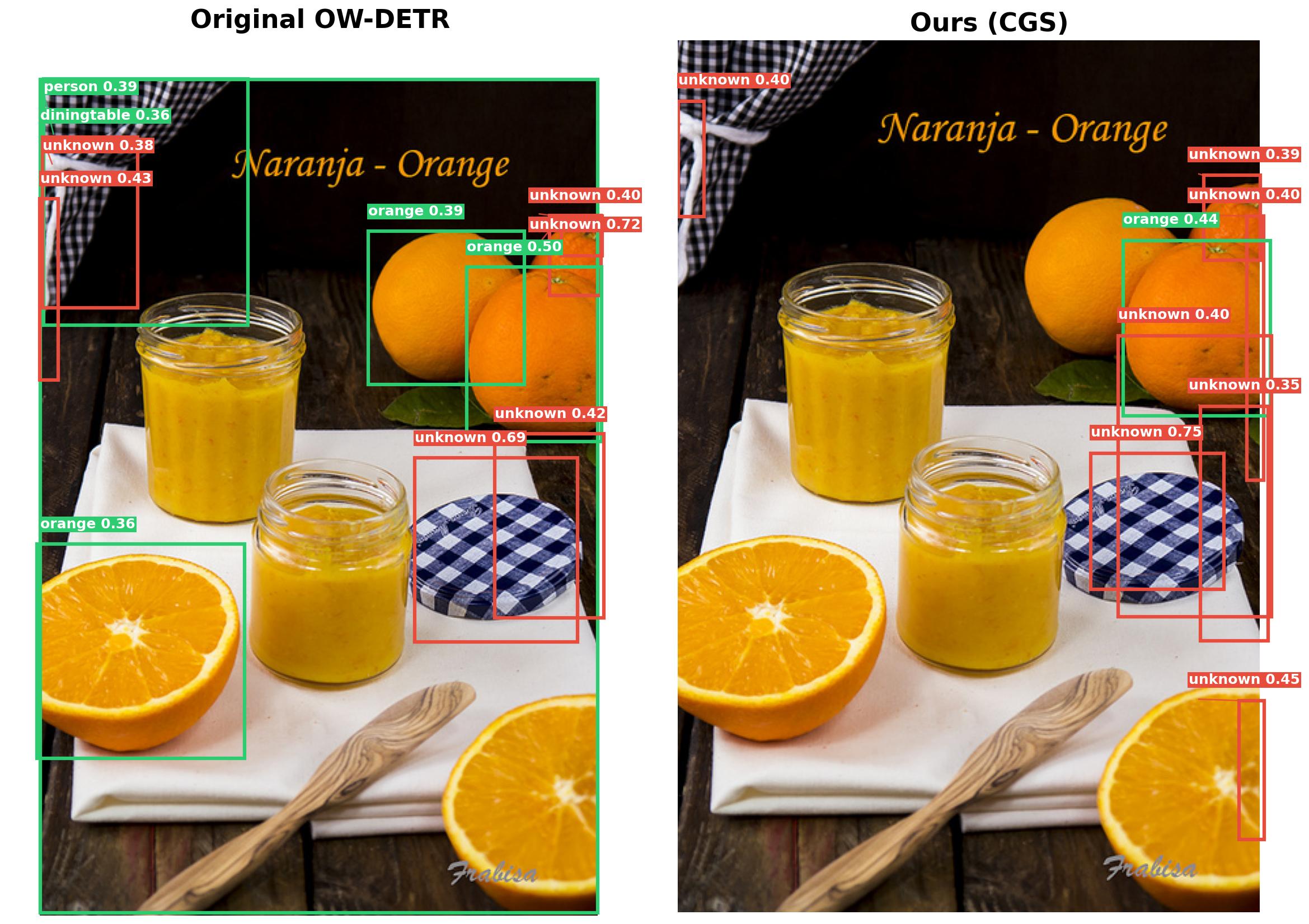}\\
            Additional known objects recovered. & Missed known object and more unknown boxes.\\
            \includegraphics[width=0.4\linewidth, cfbox=red 2pt]{ 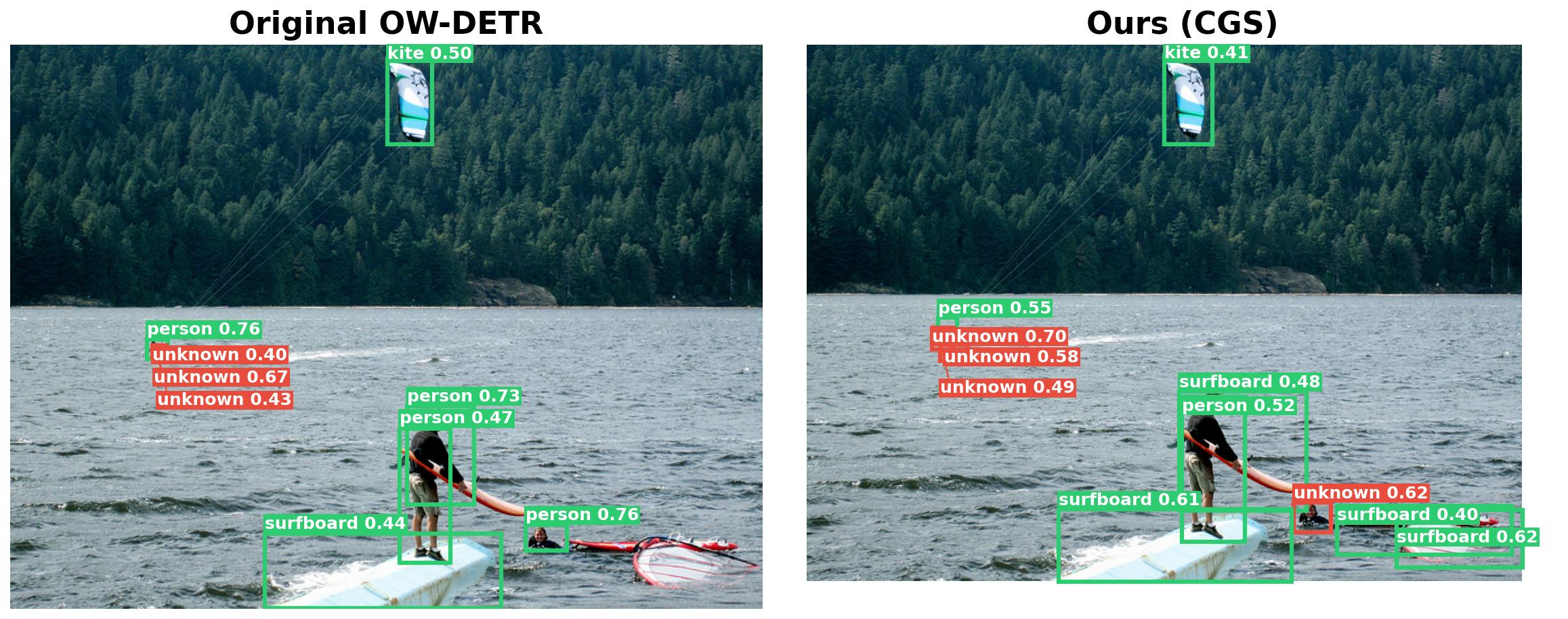} & 
            \includegraphics[width=0.4\linewidth, cfbox=red 2pt]{ Figures/qualitative_000000232088.jpg}\\
            Person region partially confused as unknown. & 
            Semantic confusion and increased unknown detections.\\
    \end{tabular}
    \caption{
    \textbf{Qualitative comparison on OWOD-COCO.}
    For each example, the left image shows original OW-DETR predictions and the right image shows predictions after class-geometry supervision (CGS).
    Green boxes denote known-class detections and red boxes denote unknown detections.
    Examples with green borders highlight representative improvements: CGS can recover known objects previously marked as unknown, correct some misclassifications, and recover additional known objects.
    Examples with red borders highlight the tradeoff: increased open-world sensitivity can introduce extra unknown boxes, partial person-to-unknown confusion, and occasional semantic label errors.
    }
    \label{fig:coco_qualitative}
\end{figure*}
\begin{figure*}[t]
    \centering
    \begin{tabular}{c}
        \includegraphics[width=0.8\linewidth, cfbox=green 2pt]{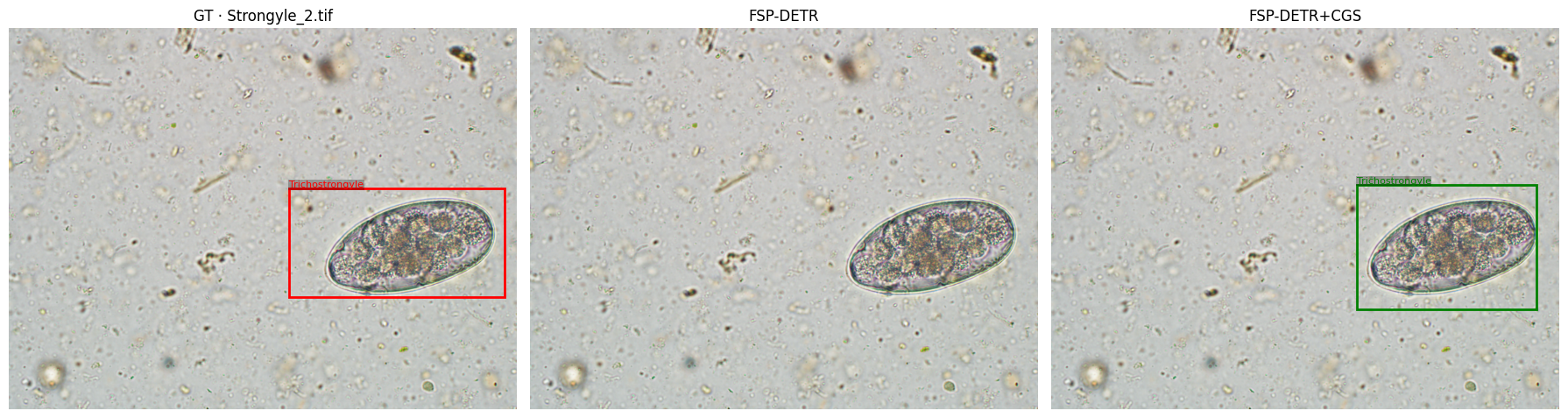} \\
        \includegraphics[width=0.8\linewidth, cfbox=red 2pt]{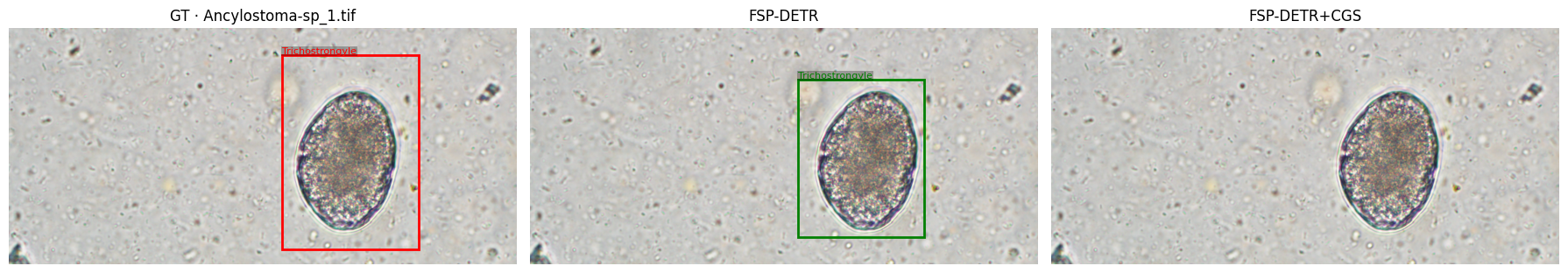} \\
        \includegraphics[width=0.8\linewidth, cfbox=green 2pt]{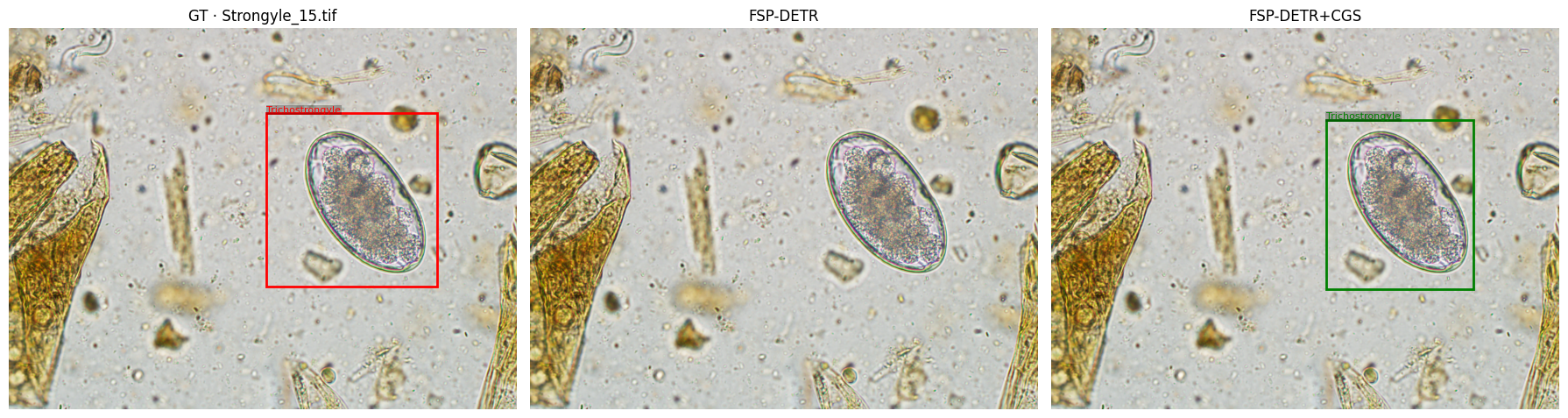} \\
        \includegraphics[width=0.8\linewidth, cfbox=green 2pt]{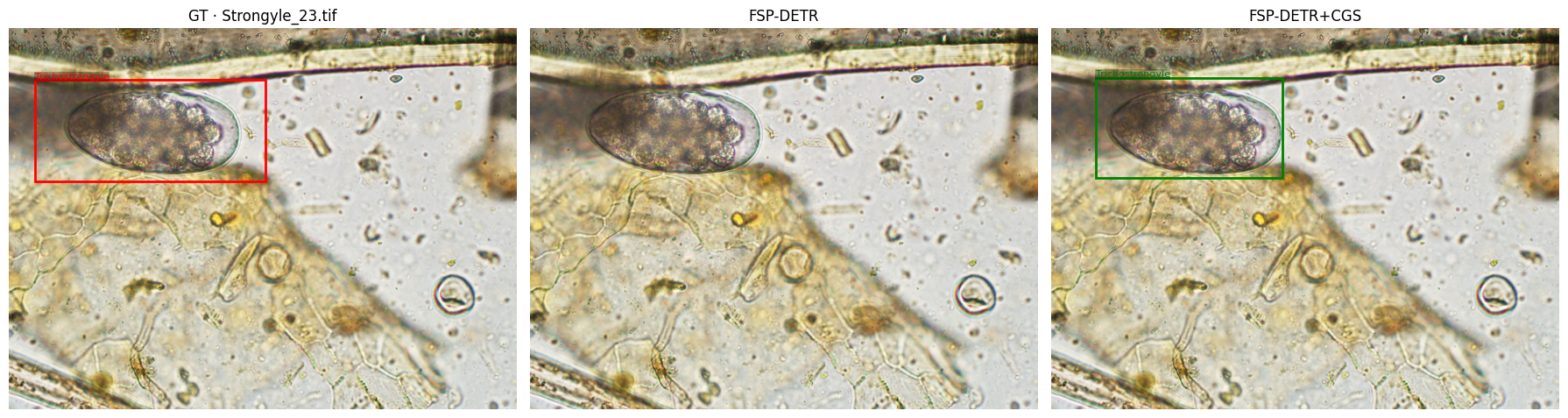} \\
    \end{tabular}
    \caption{
    \textbf{Qualitative comparison on Ova dataset.}
    For each example, the left image shows ground truth, the middle shows FSP-Detr predictions, and the right image shows predictions after class-geometry supervision (CGS). The red bounding box denotes the ground truth ova parasite, and the green bounding boxes denote the predictions made by FSP-Detr and FSP-Detr+CGS. The examples with green borders highlight the improvement of FSP-Detr with class geometry supervision (CGS). The red borders denote the instances where CGS did not improve FSP-Detr prediction. 
    }
    \label{fig:ova_qualitative}
\end{figure*}
\subsubsection{Qualitative analysis.}
Figure~\ref{fig:coco_qualitative} shows that class-geometry supervision changes the operating behavior of the detector in a consistent way.
In several cases, CGS makes the model less likely to absorb unfamiliar regions into incorrect known classes, instead either recovering the correct known category or isolating ambiguous content as unknown.
This is visible in examples where a bottle is no longer rejected as unknown, a bear is no longer confused with a car, and additional sports-related objects are correctly recognized.
However, the same behavior also produces a failure mode: the model becomes more sensitive to ambiguous regions and can over-predict unknown objects, including on parts of known objects or humans.
Thus, the qualitative results mirror the quantitative tradeoff in Table~4: CGS improves unknown awareness, but may also reduce precision through additional unknown detections and occasional semantic confusion.
Figure~\ref{fig:ova_qualitative} shows similar qualitative visualizations on the few-shot Ova detection problem. 
\section{Supplementary Implementation Details}
\label{supp:implementation}

\paragraph{Prototype recognition.}
We train ProtoNet using the same episodic train/validation/test splits as the baseline. Each episode samples \(N=\) 16 classes with \(K\in\{5,10,15,20\}\) support examples per class and 5 query examples per class. The encoder is Wide Resnet, trained with Adam optimizer, learning rate $1 \times 10^{-5}$, and episode size 500 for 100 epochs. CGS is applied to the episode-level class prototypes with \(\lambda=\{\mathrm{0.05, 0.01, 0.1, 0.5, 1}\}\), selected on the validation split.

\paragraph{Ova detection.}
For few-shot ova detection, we use the same 15 known-class / 5 held-out-class split, image preprocessing, and support construction as FSP-DETR. DETR is first trained with the standard detection objective, after which support-derived prototypes are computed from \(K\in\{5,10,15,20\}\) examples per class. CGS is applied during the prototype-matching stage with \(\lambda=\{\mathrm{0.05, 0.01, 0.1, 0.5, 1}\}\). Visual dissimilarities \(D_{\mathrm{vis}}\) are computed from cropped ground-truth support regions resized to \(224\times 224\), using average pairwise pixel MSE between classes. For text geometry, class descriptions are generated using Google Gemini and embedded using MPNet; pairwise text dissimilarity is computed as \(1-\cos(e_i,e_j)\).

\paragraph{Open-set and novel-class insertion.}
Unknown rejection uses the nearest-prototype distance \(m(z)=\min_c d(z,p_c)\), with threshold \(\tau_{\mathrm{unk}}\) selected on the validation split. For novel-class insertion, prototypes for the held-out classes are computed from their support examples and added to the existing prototype set without retraining existing class prototypes unless otherwise stated. The expanded geometry matrix includes pairwise relationships between seen and newly inserted classes.

\paragraph{COCO OWOD adaptation.}
For COCO, we follow the OW-DETR task protocol and report results using the standard Task~1--4 evaluation format. In frozen-head adaptation, OW-DETR features are fixed and only a two-layer classification MLP is fine-tuned with CGS. In full fine-tuning, CGS is added to the standard OW-DETR training objective and all model parameters are updated. For Fig.~2(c), Task~1 uses \(K\in\{5,10,15,25,50,100,250,500,750,1000\}\) labeled examples per class and reports relative change with respect to the full OW-DETR Task~1 reference.

\paragraph{Random-geometry ablation.}
The random-geometry baseline preserves the regularization form while removing meaningful class-pair structure. We construct \(D_{\mathrm{rand}}\) by randomly permuting the off-diagonal entries of \(D_{\mathrm{vis}}\), then symmetrizing the matrix and setting the diagonal to zero. This preserves the marginal distribution of pairwise dissimilarities while destroying their assignment to specific class pairs.

\paragraph{Optimization and model selection.}
All thresholds, CGS weights, and normalization choices are selected using the validation split. Unless otherwise specified, pairwise distances in the prototype space use squared euclidean distance, and off-diagonal entries of \(D\) and \(G\) are standardized before computing \(\mathcal{L}_{\mathrm{CG}}\). We report the best validation-selected model for each setting and use the same evaluation protocol as the corresponding baseline.

\paragraph{Hardware.}
Experiments are run on a server with Nvidia RTXA550 GPUs. Ova recognition and detection experiments use 1 GPU with 24 GB memory. COCO OWOD adaptation and full fine-tuning experiments use 4 GPUs. Training time is approximately 6 hours for ProtoNet recognition, 25 minutes for ova detection, and 10 hours for COCO adaptation.
\end{document}